\documentclass[]{spie}  
\usepackage[]{graphicx}
\usepackage[]{amsmath}
\usepackage[]{amssymb}
\usepackage[]{caption}
\usepackage[]{algorithm2e}
\usepackage[]{subfig} 
\usepackage[]{multirow}
\usepackage[]{array}

\DeclareMathOperator*{\argmin}{arg\,min}
\DeclareMathOperator*{\dvg}{\operatorname{div}}

\title{A Continuous Max-Flow Approach to General Hierarchical Multi-Labeling Problems}

\author{John S.H. Baxter\supit{a,b}, Martin Rajchl\supit{a,b},
	Jing Yuan\supit{a,b}, and Terry M. Peters\supit{a,b}
\skiplinehalf
\supit{a}Robarts Research Institute, London, Ontario, Canada; \\
\supit{b}Western University, London, Ontario, Canada
}

\authorinfo{Send correspondence to J.S.H.B.: E-mail: jbaxter@robarts.ca}

\begin{document} 
  \maketitle 

\begin{abstract}
Multi-region segmentation algorithms often have the onus of incorporating complex anatomical knowledge representing spatial or geometric relationships between objects, and general-purpose methods of addressing this knowledge in an optimization-based manner have thus been lacking. This paper presents Generalized Hierarchical Max-Flow (GHMF) segmentation, which captures simple anatomical part-whole relationships in the form of an unconstrained hierarchy.  Regularization can then be applied to both parts and wholes independently, allowing for spatial grouping and clustering of labels in a globally optimal convex optimization framework. For the purposes of ready integration into a variety of segmentation tasks, the hierarchies can be presented in run-time, allowing for the segmentation problem to be readily specified and alternatives explored without undue programming effort or recompilation.
\end{abstract}


\section{INTRODUCTION}
\label{sec:intro}
Multi-region segmentation problems are becoming increasingly common in medical imaging, whether using multiple areas of interest to develop nuanced metrics for computer-assisted diagnosis, or to provide context in image-guided interventions. However, the segmentation of multiple regions simultaneously has traditionally been a difficult problem, especially when the regions have defined geometric or spatial relationships with each other, which can be considered as an abstract form of anatomical knowledge. Traditional approaches to multi-region segmentation have been primarily model-based or atlas-based, both of which require a large number of prior segmented images to be of use and often have difficulty adapting to unanticipated or unpredictable pathologies. 

Recently, optimization approaches have arisen to tackle multi-region segmentation problems, notably discrete graph-cuts\cite{boykov_fast_2001}. Variational and continuous counterparts have since arisen to handle problems of metrification and stair-case artifacts unavoidable in discrete methods \cite{yuan_study_2010}. These algorithms minimize an energy functional subject to constraints which can represent anatomical knowledge, optimizing all regions in the image in tandem. Methods for extending these approaches to an arbitrary number of labels have been proposed \cite{pock_convex_2008,yuan_continuous_2010,bae_fast_2011} but these either do not allow for the specification of abstract anatomical knowledge, or constrain the addressed problems to those fitting a particular geometric form. Relaxing those constraints in the discrete form while maintaining global optimality has been addressed by Delong et al. \cite{delong2009globally} by creating label hierarchies with containment and exclusion (partition) operators that preserve the submodularity of the energy functional.

The motivation behind this work is to extend these hierarchies to the continuous case, using variational optimization to optimally segment an image into multiple regions and hierarchical label orderings to provide some abstract anatomical knowledge into how those regions interact.

\section{Contributions}
This paper proposes a continuous max-flow formulation which addresses a hierarchical multi-labeling problem. We address this by building a novel continuous max-flow model which scales based on an input hierarchy. We can then show the equivalence between this formulation and the convex-relaxation of a continuous min-cut formulation under hierarchical constraints.

This algorithm displays a high degree of parallelism within each optimization step, allowing for acceleration through general purpose graphic processing unit (GPGPU) computation, as well as potential concurrency between optimization steps allowing for additional threading and scheduling to improve performance and allowing for multi-GPU use.

\section{Convex relaxed hierarchical models and Previous Work}

\subsection{Previous Work}
Previous work by Yuan et al. \cite{yuan_study_2010} has addressed the continuous binary min-cut problem:
\begin{gather*}
E(u) = \int\limits_{\Omega}(D_s(x)u(x) + D_t(x)(1-u(x))+ S(x)|\nabla u(x)|)dx \\
\mbox{ s.t. } u(x) \in \{ 0, 1 \}
\end{gather*}
as well as the convex relaxed continuous Potts Model:
\begin{gather*}
E(u) = \sum\limits_{\forall L} \int\limits_{\Omega}(D_L(x)u_L(x)+ S(x)|\nabla u_L(x)|)dx \\
\mbox{ s.t. } u_L(x) \geq 0 \mbox{ and } \sum\limits_{\forall L}u_L(x) = 1 \mbox{ .}
\end{gather*}
These techniques both used a continuous max-flow model with augmented Lagrangian multipliers. In the case of the convex-relaxed continuous Potts model, the source flow had infinite capacity, the costs in the functional corresponding with constraints on the sink flows.

Bae et al. \cite{bae_global_2011} extended the work on the  continuous binary min-cut problem to the continuous Ishikawa model:
\begin{gather*}
E(u) = \sum\limits_{L=0}^N \int\limits_{\Omega}(D_L(x)u_L(x)+ S(x)|\nabla u_L(x)|)dx \\
\mbox{ s.t. } u_L(x) \in \{ 0,1 \} \mbox{ and } u_{L+1}(x) \leq u_L(x)
\end{gather*}
using similar variational methods but a tiered continuous graph analogous to that used by Ishikawa \cite{ishikawa_exact_2003} in the discrete case, that is, with finite capacities on intermediate flows between labels.

These models have since been extended to incorporate star-shaped constraints on the various labels\cite{yuan_efficient_2012}. The convex-relaxed continuous Potts model has also been extended to encorporate a limited hierarchical constraint by Rajchl et al.\cite{rajchl_interactive_2014} for myocardial scar segmentation.

\subsection{Convex relaxed hierarchical models}
Hierarchical models are a general extension of both Potts and Ishikawa models. As with those models, the problem can be expressed as an optimization problem with the given objective function:
\begin{equation}
\underset{ \{\Omega_L\} }{\min} \text{ }  E = \sum\limits_{\forall L} \left( \int_{\Omega_L} D_L(x)dx+ \int_{\delta \Omega_L} S_{L}(x) dx \right)
\end{equation}
where $D_L(x) \geq 0$ and $S_L(x) \geq 0$ weight the interior and boundary of $L$ as a function of position $x$. But in the case of hierarchical models, the sets, $\Omega_L$ are not entirely disjoint (as in the Potts model) or are subset of each other (as in the Ishikawa model) but can be arranged in a hierarchy, this hierarchy can be thought of as a rooted tree. For the sake of notation, we will refer to the `parent' of a label $L$ as $L.P$. The `leaves' of the hierarchy are the set of labels with no children, that is $\mathcal{L} = \left\{ L | L.C = \emptyset \right\}$. The set corresponding to the leaf labels forms a partition of the entire image as in the constraints of the Potts model, that is:
\begin{equation}
\forall L_1,  L_2 \in \mathcal{L}, \text{ } ( L_1 \neq L_2 \implies \Omega_{L_1} \cap \Omega_{L_2} = \emptyset) 
\end{equation}
and
\begin{equation}
\bigcup_{L \in \mathcal{L}} \Omega_L = \Omega \text{ .}
\label{eq:totality}
\end{equation}
Parent labels are simply the union of their children, that is:
\begin{equation}
\forall L \not\in \mathcal{L}, \bigcup_{L' \in L.C} \Omega_{L'} = \Omega_L \text{ .}
\end{equation}
Other than those, the remaining constraints should ensure that the hierarchy is a valid rooted tree. Specifically, a unique root node exists ($!\exists S (\not\exists S.P) $), the $.P$ and $.C$ operators are consistent and do not form any cycles and that the graph is connected. (One can rewrite equation \eqref{eq:totality} as $\Omega_S = \Omega$.) Note that these constraints are based on the hierarchy rather than the solution space being optimized over. An example of a hierarchy is presented in Figure \ref{fig:lgeMriScar}, adapted from Rajchl et al. \cite{rajchl_interactive_2014} In this example, the leaves are the thoracic background (T), blood pool (B), healthy myocardium (M), and myocardial scar (Sc). The last three labels are all children of a common cardiac (C) label which encourages their spatial grouping.

\begin{figure}[h]
\centering
\includegraphics[width=60mm]{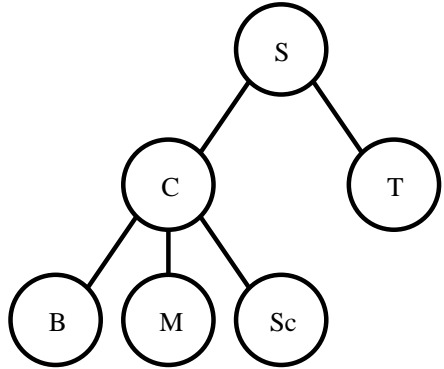}
\caption{Example Segmentation Hierarchy}
\label{fig:lgeMriScar}
\end{figure}

This formulation is relatively powerful in that they can express both Potts and Ishikawa models in common use. To represent the Potts model, the hierarchy should consist only of the root and the leaves with no additional vertices. Any Ishikawa model can be represented using a hierarchy where every parent node bifurcates, that has exactly two children, and at least one of them is a leaf. In that sense, one can think of a general class of multi-labeling problems, the Generalized Hierarchical class, which contains both the class of Potts model problems and Ishikawa model problems as strict sub-classes. (This proofs are provided in Section \ref{sec:problemClass}).

The first step in the convex relaxation of these models is to represent each label's spatial set, $\Omega_L$, as a labeling function, $u_L(x) \in [0,1]$. As in \cite{yuan_study_2010, yuan_continuous_2010, bae_global_2011}, the labeling function has the following properties:
\begin{equation}
u_L(x) = \left\{ \begin{array}{ll}
 1, & x \in \Omega_L \\ 
 0, & x \not\in \Omega_L \\
 \end{array} \right.
\end{equation}
\begin{equation}
\int_{\delta \Omega_L}S_L(x)dx = \int_\Omega S_L(x) |\nabla u_L(x)| dx
\end{equation}
Alternatively, these labeling functions can be interpretted as a fuzzy or probabilistic segmentation.

These properties yield the convex relaxed generalized hierarchical model:
\begin{equation}
\begin{aligned}
& \underset{u}{\text{minimize}}
& & E(u) = \sum\limits_{\forall L} \int_\Omega \left( D_L(x)u_L(x)+S_{L}(x)|\nabla u_L(x)| \right) dx \\
& \text{subject to} & & \forall L (u_L(x) \geq 0) \\
& & &  \forall L \left(\sum\limits_{ L' \in L.C }u_{L'}(x) = u_L(x)\right) \\
& & &   u_S(x)=1
\end{aligned}
\label{ghmf_form}
\end{equation}
which can be solved with global optimality for probabilistic labeling and approximated for discrete labels.

\subsection{Generalized Hierarchical Class} \label{sec:problemClass}
As stated earlier, the class of convex relaxed problems solvable through GHMF encompasses those solvable by both Potts and Ishikawa models via a polynomial time reduction.

\begin{theorem} \label{PottsReduced}
Any instance of the Potts formulation:
\begin{equation*}
\begin{aligned}
& \underset{u}{\text{minimize}}
& & E(u) = \sum\limits_{\forall L} \int_\Omega \left( D_L(x)u_L(x)+S(x)|\nabla u_L(x)| \right) dx \\
& \text{subject to}
& & \forall L (u_L(x) \geq 0) \\
& & & \sum\limits_{ \forall L }u_L(x) = 1
\end{aligned}
\end{equation*}
can be polynomial time reduced to an instance of the GHMF formulation:
\begin{equation*}
\begin{aligned}
& \underset{u}{\text{minimize}}
& & E(u) = \sum\limits_{\forall L} \int_\Omega \left( D_L(x)u_L(x)+S_{L}(x)|\nabla u_L(x)| \right) dx \\
& \text{subject to} & & \forall L (u_L(x) \geq 0) \\
& & &  \forall L \left(\sum\limits_{ L' \in L.C }u_{L'}(x) = u_L(x)\right) \\
& & &   u_S(x)=1
\end{aligned}
\end{equation*}
\end{theorem}

\begin{proof}[Proof of Theorem \ref{PottsReduced}]
Consider a hierarchy in which the root node, $S$, is immediately partitioned into a set of leaf nodes, each corresponding with a label, $L$, from the Potts model. Note that this formulation makes the Potts model constraint
\begin{equation*}
\sum\limits_{ \forall L }u_L(x) = 1
\end{equation*}
equivalent to the combination of GHMF constraints
\begin{equation*}
\forall L' \left(\sum\limits_{ L \in \{L | L.P = L'\} }u_L(x) = u_{L'}(x)\right) \text{ and } (\not \exists S.P \implies u_S(x)=1) \wedge \exists ! S (\not \exists S.P) \text{ .}
\end{equation*}

Knowing this, we can directly translate the data terms from the Potts model to the GHMF model without change, setting $D_S(x)=0$. The smoothness terms can be copied directly, setting $S_L(x)=S(x)$ for all leaf nodes, and $\alpha_S=S_S(X)=0$ for the root node. This makes the two formulations equivalent. Both the transformations and the construction of the hierarchy, are obviously possible in polynomial time, the output requiring no transformation. \hfill
\end{proof}

\begin{theorem} \label{IshikawaReduced}
Any instance of the Ishikawa formulation with levels $[0 .. N]$:
\begin{equation*}
\begin{aligned}
& \underset{u}{\text{minimize}}
& & E(u) = \sum\limits_{\forall L} \int_\Omega \left( D_L(x)u_L(x)+ S_i(x)|\nabla u_L(x)| \right) dx \\
& \text{subject to}
& & \forall i ( u_{L_i}(x) \geq 0) \\
& & & \sum\limits_{ \forall i \in [1 .. N] }u_{L_i}(x) \leq u_{L_{i-1}}(x) \\
& & & u_{L_0}(x) = 1
\end{aligned}
\end{equation*}
with output labelling, $u_{L_i}(x)$ for $i\in[1..N]$ can be polynomial time reduced to an instance of the GHMF formulation:
\begin{equation*}
\begin{aligned}
& \underset{u}{\text{minimize}}
& & E(u) = \sum\limits_{\forall L} \int_\Omega \left( D_L(x)u_L(x)+ S_{L}(x)|\nabla u_L(x)| \right) dx \\
& \text{subject to} & & \forall L (u_L(x) \geq 0) \\
& & &  \forall L \left(\sum\limits_{ L' \in L.C }u_{L'}(x) = u_L(x)\right) \\
& & &   u_S(x)=1
\end{aligned}
\end{equation*}
\end{theorem}

\begin{proof}[Proof of Theorem \ref{IshikawaReduced}]
Consider introducing the dummy labels, $B_i$ for $i\in[1..N]$. Define the labeling function for each of these labels as:
\begin{equation*}
u_{B_i}(x) = u_{L_{i-1}}(x) - u_{L_i}(x) .
\end{equation*}

Note that $u_{B_i}(x) \geq 0$ since $u_{L_i}(x) \leq u_{L_{i-1}}(x)$ is a constraint on the Ishikawa model.
Let us now construct a hierarchy where each branch node $L_i$ where $i\in[0..N-1]$ bifurcates into the node $L_{i+1}$ and leaf node $B_{i+1}$. The label $L_N$ is distinguished from the other non-dummy labels by being a leaf. This hierarchy can be constructed in $O(N)$ time.

For the data terms, we associate every dummy label with the data term $D_{B_i}(x)=0$ and copy over the data terms for the labels $L_i$ from the Ishikawa model. We use the same policy for the smoothness terms, setting $S_{B_i}(x)=0$ and $S_{L_i}(x)=S_i(x)$. Note that this formulation is equivalent to the Ishikawa model, meaning that it will produce the same values of $u_{L_i}$. Note that the definition of $u_{B_i}(x)$ and constraints from the Ishikawa model are together equivalent to the constraints from the GHMF model.

Lastly, we must transform the output from the GHMF model into the equivalent output from the Ishikawa model. First, we should note that the output labeling $u_{L_N}(x)$ is the same for both and can be copied over. Starting from $i=N-1$ and moving to $i=1$, we can compute the output from the Ishikawa model from the output of the GHMF model using the definition of $u_{B_i}(x)$. That is, we construct the Ishikawa output labeling $u_{L_i}(x)$ as $u_{L_{i+1}}(x) + u_{B_{i+1}}(x)$. This construction can occur in $O(XN)$ time, meaning each reduction is polynomial time. \hfill
\end{proof}

\subsection{Data Term Structure}
Without loss of generality, we would like to constrain the data terms, $D_L(x)$, to be non-zero only at the leaves, at non-negative at that. To do this, we would like to show that this constraint does not limit the class of problems handled by this algorithm, and that such a constraint can be implemented in linear time.

\begin{theorem} Data Pushdown Theorem: \label{dataPushdown}\newline
The GHMF formulation:
\begin{equation*}
\begin{aligned}
& \underset{u}{\text{minimize}}
& & E(u) = \sum\limits_{\forall L} \int_\Omega \left( D_L(x)u_L(x)+S_{L}(x)|\nabla u_L(x)| \right) dx \\
& \text{subject to}
& & \forall L (u_L(x) \geq 0) \\
& & &  \forall L \left(\sum\limits_{ L' \in L.C }u_{L'}(x) = u_L(x)\right) \\
& & &  u_S(x) = 1 
\end{aligned}
\end{equation*}
can be polynomial time reduced to an instance of the GHMF formulation with two additional constraints on the input:
\begin{equation*}
\begin{aligned}
& & &  \forall{L}(D_L(x) \geq 0) \\
& & &  |L.C| \neq 0 \implies D_L(x) = 0
\end{aligned}
\end{equation*}
\end{theorem}

\begin{proof}[Proof of Theorem \ref{dataPushdown}] 
Consider the hierarchy, $H$, with a non-leaf node, $L'$. Assume that $L'$ has a non-zero data term associated with it, $D_{L'}(x)$. We can express the objective function as:

\begin{equation*}
\begin{aligned}
& \underset{u}{\text{minimize}}
& & E(u) = \int_\Omega D_{L'}(x)u_{L'}(x)dx + 
\sum\limits_{\forall L \in L'.C} \int_\Omega D_L(x)u_L(x)dx+ \\
& & & \sum\limits_{\forall L \notin L'.C \cup \{L'\}} \int_\Omega D_L(x)u_L(x)dx+ 
\sum\limits_{\forall L} \int_\Omega S_{L}(x)|\nabla u_L(x)| dx \\
\end{aligned}
\end{equation*}

Note that using the second constraint, we can break up the data term associated with the node $L'$ as follows:
\begin{equation*}
\begin{aligned}
\int_\Omega D_{L'}(x)u_{L'}(x)dx +  \sum\limits_{\forall L \in C} \int_\Omega D_L(x)u_L(x)dx 
& = \int_\Omega D_{L'}(x)\left( \sum\limits_{ L \in C }u_L(x) \right)dx +  \sum\limits_{\forall L \in C} \int_\Omega D_L(x)u_L(x)dx \\
& = \sum\limits_{\forall L \in C} \int_\Omega D_{L'}(x)u_L(x)dx +  \sum\limits_{\forall L \in C} \int_\Omega D_L(x)u_L(x)dx \\
& = \sum\limits_{\forall L \in C} \int_\Omega \left(D_L(x)+D_{L'}(x)\right)u_L(x)dx
\end{aligned}
\end{equation*}
We can construct an equivalent problem with data terms, $D'_L(x)$, where:
\begin{equation*}
\begin{aligned}
 D'_{L'}(x) = 0
\end{aligned}
\end{equation*}
\begin{equation*}
\begin{aligned}
 \forall L \in C \left( D'_L(x) = D_L(x)+D_{L'}(x) \right)
\end{aligned}
\end{equation*}
\begin{equation*}
\begin{aligned}
 \forall L \notin C \cup \{{L'}\} \left( D'_L(x)=D_L(x) \right)
\end{aligned}
\end{equation*}

Therefore, we can eliminate the data term for single non-leaf node in $O(NX)$ time by pushing it down to the leaves. We can apply this pushdown procedure in a pre-order traversal, which would ensure that all data terms are pushed down to the leaves using $O(N)$ pushdown operations. Since the optimization formulae are equal, the output does not need to be modified. Therefore, the formulation without the additional constraint can be reduced to that with the additional constraint that only leaf nodes have non-zero data terms. 

The last constraint is trivially held for all branch data terms under the first polynomial time reduction, but we still have to ensure that the leaf node data terms are non-negative. Consider $D_L(x)$ as the data terms after the first polynomial time reduction. We can construct an equivalent set of non-negative data terms by adding a constant value to $D_L(x)$ at each $x$ that makes each term non-negative. The new data terms would therefore be equal to:
\begin{equation*}
D'_L(x) = D_L(x) + \underset{L'}\min{D_{L'}(x)} .
\end{equation*}

This polynomial-time modification only adds a constant term, $ \int_\Omega\underset{L'}\min{D_{L'}(x)}dx$, to the value of $E(u)$ meaning that it does not change the optimal labeling, thus showing that the two formulations are equivalent. \hfill
\end{proof}

This reduction is obviously optimal in terms of asymptotic complexity since it takes the equivalent time as its verification, specifically, the linear time required to ensure that the data terms are non-negative over all voxels.

\section{Continuous Max-Flow Model}
\subsection{Primal Formulation}
The modeling approach is derived from those presented in \cite{bae_global_2011, bae_fast_2011, rajchl_fast_2012, yuan_study_2010, yuan_continuous_2010} and follows along the same format, using duality through an augmented Lagrangian formulation. The primal model represents network flow maximization through a large graph with only the sink flows constrained. The dual of this formulation is the GHMF equation \eqref{ghmf_form} as we shall prove in this section. We can write the primal model as:
\begin{equation}
\label{primal}
\begin{aligned}
\underset{p,q} \max & \int_\Omega p_S(x) dx
\end{aligned}
\end{equation}
\begin{subequations}
subject to the capacity constraints:
\begin{align}
p_L(x) & \leq D_L(x), \text{ } L \in \text{ leaves}\\
|q_L(x)| & \leq S_L(x) \text{ } L \neq S
\end{align}
and the flow conservation constraint:
\begin{align}
0 =\operatorname{div} q_L(x) + p_L(x) - p_{L.P}(x) \text{ .}
\end{align}
\end{subequations}

This is equivalent to a multi-flow problem over a large graph constructed from the image dimensions and the provided hierarchy. The only constraints placed on the capacities in the graph are the spatial constraints limiting the magnitude of the spatial flows, and constraints on the flows from the leaf labels to the sink. The remaining flows, specifically the flows between labels, are assumed to be of infinite capacity.

\subsection{Primal-Dual Formulation}
The primal model can be converted to a primal-dual model through the use of Lagrangian multipliers on the flow conservation constraint $G_L(x) = \operatorname{div} q_L(x) + p_L(x) - p_{L.P}(x) = 0$. This yields the equation:
\begin{equation}
\label{lagrangian}
\begin{aligned}
\underset{u} \min \underset{p,q} \max & \left( \int_\Omega p_S(x) dx + \sum_{\forall L \neq S} \int_\Omega u_L(x) G_L(x) dx \right) \\
p_L(x) & \leq D_L(x), \text{ } L \in \text{ leaves}\\
|q_L(x)| & \leq S_L(x) \text{ } L \neq S \text{ .}
\end{aligned}
\end{equation}

To ensure that this function meets the criteria of the minimax theorem, we must ensure that it is convex with respect to $u$, considering $p,q$ to be fixed, and concave with respect to $p,q$ with $u$ fixed. \cite{ekeland_convex_1976} For the first, it is sufficient to note that if $p,q$ are fixed, then $G$ is fixed as well, meaning that \eqref{lagrangian} is linear and therefore convex with respect to $u$. It should also be noted that $G$ is a linear function of $p,q$, meaning that \eqref{lagrangian} is again linear and therefore concave with respect to $p,q$, confirming the existance of a saddle point and the equivalence of the formulation regardless of the order of the prefix max and min operators. \cite{ekeland_convex_1976}

\subsection{Dual Formulation}
To show the equivalence of the primal-dual formulation to the convex relaxed generalized hierarchical model, we can consider the optimization of each set of flows. We can find the saddle point through the optimization of the sink-flows, $p_L$, working bottom-up and the spatial flows within each label. Starting with any leaf label, $L$, we can isolate $p_L$ in \eqref{lagrangian} giving
\begin{equation}
\underset{u} \min \underset{p_L(x) \leq D_L(x)} \max \int_\Omega u_L(x)p_L(x)dx = \underset{u, u_L(x) \geq 0} \min  \int_\Omega u_L(x)D_L(x)dx
\end{equation}
when $u_L(x) \geq 0$. (If $u_L(x) < 0$, the function can be arbitrarily maximized by $p_L(x) \to -\infty$.) Working upwards, every branch label, $L$, can be isolated in \eqref{lagrangian} as
\begin{equation}
\underset{u} \min \underset{p_L(x)} \max \left( \int_\Omega u_L(x)p_L(x)dx - \sum_{\forall L' \in L.C} \int_\Omega u_{L'}(x)p_L(x)dx \right) = 0
\end{equation}
at the saddle point defined by $u_L(x) =  \sum_{\forall L' \in L.C} u_{L'}(x)$. Lastly, the source flow, $p_S$, can be isolated in a similar manner, that is:
\begin{equation}
\underset{u} \min \underset{p_S(x)} \max \left( \int_\Omega p_S(x)dx - \sum_{\forall L' \in S.C} \int_\Omega u_{L'}(x)p_S(x)dx \right) = 0
\end{equation}
at the saddle point defined by $1 =  \sum_{\forall L' \in S.C} u_{L'}(x)$. These constraints combined yield the labeling constraints in the original formulation. The maximization of the spatial flow functions can be expressed in a well-studied form \cite{giusti_minimal_1984} as:
\begin{equation} \label{eq:regularSpatialFlow}
\begin{aligned}
\int_\Omega u_L(x) \dvg q_L(x)dx
& = \int_\Omega \left( \dvg (u_L(x)q_L(x)) - q(x) \cdot \nabla u_L(x) \right) dx \\
& = \int_\Omega \dvg (u_L(x)q_L(x)) dx  - \int_\Omega  q_L(x) \cdot \nabla u_L(x) dx \\
& = \oint_{\delta\Omega} u_L(x)q_L(x) \cdot d\mathbf{s}  - \int_\Omega  q_L(x) \cdot \nabla u_L(x) dx \\
& = - \int_\Omega  q_L(x) \cdot \nabla u_L(x) dx \\
\underset{|q_L| \leq S_L(x)} \max \int_\Omega u_L(x) \dvg q_L(x)dx
& = \underset{|q_L| \leq S_L(x)} \max  - \int_\Omega  q_L(x) \cdot \nabla u_L(x) dx \\
& =  - \int_\Omega  \left( -\frac{S_L(x)}{|\nabla u_L(x)|} \nabla u_L(x) \right) \cdot \nabla u_L(x) dx \\
& = \int_\Omega S_L(x) |\nabla u_L(x)| dx \\
\underset{u_L \geq 0} \min \underset{|q_L| \leq S_L(x)} \max \int_\Omega u_L(x) \dvg q_L(x)dx 
& = \underset{u_L \geq 0} \min \int_\Omega S_L(x) |\nabla u_L(x)| dx \text{ .}
\end{aligned}
\end{equation}
Meaning that we can express the saddle point of equation \eqref{lagrangian} as the original energy functional, \eqref{ghmf_form} and therefore, finding the saddle point of \eqref{lagrangian} is equivalent to solving the GHMF segmentation problem.

\newpage
\section{Solution to Primal-Dual Formulation}
To address the optimization problem, We can find this saddle point by augmenting the Lagrangian function \cite{bertsekas_nonlinear_1999}:
\begin{equation}
\label{augmented}
\begin{aligned}
\underset{u} \min \underset{p,q} \max & \left( \int_\Omega p_S(x) dx + \sum_{\forall L \neq S} \int_\Omega u_L(x) G_L(x) dx - \frac{c}{2} \sum_{\forall L \neq S} \int_\Omega G_L(x)^2 dx \right)\\
p_L(x) & \leq D_L(x), \text{ } L \in \text{ leaves}\\
|q_L(x)| & \leq S_L(x) \text{ } L \neq S
\end{aligned}
\end{equation}
where $c > 0$ is an additional positive parameter penalizing deviation from the flow conservation constraint. Using this formula, we can maximize each component individually and iteratively. We use the following steps iteratively:
{
\begin{enumerate}
\setlength{\belowdisplayskip}{2pt} \setlength{\belowdisplayshortskip}{2pt}
\setlength{\abovedisplayskip}{2pt} \setlength{\abovedisplayshortskip}{2pt}
\item Maximize \eqref{augmented} over $q_L$ at all levels by:
\begin{equation*}
q_L(x) \gets \operatorname{Proj}_{|q_L(x)| \leq S_L(x) } \left( q_L(x) + \tau \nabla \left( \dvg q_L(x) + p_L(x) - p_{L.P}(x) - u_L(x)/c \right) \right) 
\end{equation*}
which is a Chambolle projection iteration with descent parameter $\tau > 0$. \cite{chambolle_algorithm_2004}
\item Maximize \eqref{augmented} over $p_L$ at the leaves analytically by:
\begin{equation*}
p_L(x) \gets \min\lbrace D_L(x), p_{L.P}(x) - \dvg q_L(x) + u_L(x)/c \rbrace
\end{equation*}
\item Maximize \eqref{augmented} over $p_L$ at the branches analytically by:
\begin{equation*}
p_L(x) \gets \frac{1}{|L.C|+1} \left( (p_{L.P}(x) - \dvg q_L(x) + u_L(x)/c) + \sum_{\forall L' \in L.C}\left( p_{L'}(x) + \dvg q_{L'}(x) - u_{L'}(x)/c \right) \right)
\end{equation*}
\item Maximize \eqref{augmented} over $p_S$ analytically by:
\begin{equation*}
p_S(x) \gets \frac{1}{|S.C|} \left( 1/c + \sum_{\forall L' \in S.C}\left( p_{L'}(x) + \dvg q_{L'}(x) - u_{L'}(x)/c \right) \right)
\end{equation*}
\item Minimize \eqref{augmented} over $u_L$ analytically by:
\begin{equation*}
u_L(x) \gets u_L(x) - c \left( \dvg q_L(x) - p_{L.P}(x) + p_L(x) \right)
\end{equation*}
\end{enumerate}
}

\subsection{Generalized Hierarchical Max-Flow Algorithm}
In order to improve convergence rate, we perform an initialization step that ensures optimality for the zero-smoothness condition. When performing the specific tasks outlined in the previous section, we proceed in a bottom-up manner, optimizing the leafs (where a capacity constraint on the sink flow exists) and propagating through the branches where no such capacity constraints exist. The sequential algorithm used is:
\begin{algorithm}[!h]
InitializeSolution(S) \;
\While{not converged} {
  \For{ $\forall L \neq S$ } {
    $\forall x, q_L(x) \gets \operatorname{Proj}_{|q_L(x)| \leq S_L(x) } \left( q_L + \tau \nabla \left( \dvg q_L(x) + p_L(x) - p_{L.P}(x) - u_L(x)/c \right) \right)$ \;
  }
  UpdateSinkFlows($S$) \;
  \For{ $\forall L \neq S$ } {
    $\forall x, u_L(x) \gets u_L(x) - c \left( \dvg q_L(x) - p_{L.P}(x) + p_L(x) \right)$ \;
  }
}
\label{alg:solverSeq}
\end{algorithm}

which  makes use of the following recursive function definitions:

\begin{algorithm}[!h]
{ \bf InitializeSolution(L) } \\
\For{ $\forall L' \in L.C$} {
  InitializeSolution($L'$) \;
}
$\forall x, p_L(x) \gets \underset{L' \in \text{ leaves}}\min D_{L'}(x)$ \;
\If{ $L$ is a leaf }{
  \eIf{ $L \in \underset{L' \in \text{ leaves}}\argmin D_{L'}(x) $ } {
    $\forall x, u_L(x) \gets 1 / | \underset{L' \in \text{ leaves}}\argmin D_{L'}(x) | $ \;
  }{
    $\forall x, u_L(x) \gets 0$ \;
  }
}
\If{ $L$ is a branch }{
  $\forall x, u_L(x) \gets 0$ \;
  \For{ $\forall L' \in L.C$ } {
    $\forall x, u_L(x) \gets u_L(x) + u_{L'}(x)$ \;
  }
}
\end{algorithm}

\begin{algorithm}[!h]
{ \bf UpdateSinkFlows(L) } \\
\For{ $\forall L' \in L.C$} {
  UpdateSinkFlows($L'$) \;
}
\If{ $L$ is a leaf }{
  $\forall x, p_L(x) \gets \min\lbrace D_L(x), p_{L.P}(x) - \dvg q_L(x) + u_L(x)/c \rbrace$ \;
}
\If{ $L = S$ }{
  $\forall x, p_S(x) \gets 1/c$ \;
  \For{ $\forall L' \in S.C$ } {
     $\forall x, p_S(x) \gets p_S(x) + p_{L'}(x) + \dvg q_{L'}(x) - u_{L'}(x)/c $ \;
  }
  $\forall x, p_S(x) \gets \frac{1}{|S.C|}p_S(x) $ \;
}
\If{ $L$ is a branch }{
  $\forall x, p_L(x) \gets p_{L.P}(x) - \dvg q_L(x) + u_L(x)/c $ \;
  \For{ $\forall L' \in L.C$ } {
     $\forall x, p_L(x) \gets p_L(x) + p_{L'}(x) + \dvg q_{L'}(x) - u_{L'}(x)/c$ \;
  }
  $\forall x, p_L(x) \gets \frac{1}{|L.C|+1} p_L(x)$ \;
}
\end{algorithm}
\newpage
For the sake of conciseness, the `{\bf for} \textit{$\forall x$} {\bf do} ' loops surrounding each assignment operation have been replaced with the prefix $\forall x$ in both the algorithm and the function definitions.

\newpage
\section{Discussion and Conclusions}
In this paper, we present an algorithm for addressing continuous max-flow problems where the labels can be arranged in a hierarchy using a subset/superset relationship as the parent/child relationship. This approach generalizes both continuous Potts and Ishikawa models, and opens the way to incorporating partial ordering and spatial grouping constraints into a wide array of segmentation problems with more complex spatial anatomy. In addition, we have proven that such a solver need only address problems in which data terms are provided solely at the leaf nodes, the final labels, and nowhere else through a simple linear-time reduction. 

This solver has been implemented using the NVIDIA Compute Unified Device Architecture (CUDA) allowing for performance improvements through the use of intra-task parallelization. To further improve performance, inter-task concurrency has been exploited, allowing for multiple graphics cards to be used simultaneously on a single segmentation problem.

\acknowledgments 
The authors would like to acknowledge Dr.\ Elvis Chen and Jonathan McLeod for their invaluable discussion, editing, and technical support.


\bibliographystyle{spiebib}   
\bibliography{TechReportGHMF} 

\end{document}